\documentclass[runningheads]{llncs}
\usepackage{graphicx}
\usepackage{amssymb}
\begin{document}
\title{Transition Graph Properties of Target Class Classification\thanks{Supported by grant №21T-1B314 of the Science Committee of MESCS RA.}}
%
%
\author{Levon Aslanyan\inst{1}\orcidID{0000-0002-5354-2730
} \and
Hasmik Sahakyan\inst{2}\orcidID{0000-0002-8449-6845}}
\authorrunning{L. Aslanyan, H. Sahakyan}
%
\institute{National Academy of Sciences of Republic of Armenia (NAS RA), Yerevan 0019, Armenia\\
\email{lasl@sci.am}\\
\url{http://www.sci.am}
\and
Institute for Informatics and Automation Problems of NAS RA, Yerevan 0014, Armenia\\
\email{hsahakyan@sci.am}\\
\url{http://www.iiap.sci.am}}
\maketitle              
\begin{abstract}
Target class classification is a mixed classification and transition model whose integrated goal is to assign objects to a certain, so called target or normal class. The classification process is iterative, and in each step an object in a certain class undergoes an action attached to that class, initiating the transition of the object to one of the classes. The sequence of transitions, which we call class transitions, must be designed to provide the  final assignment of objects to the target class. The transition process can be described in the form of a directed graph, and the success of the final classification is mainly due to the properties of this graph. In our previous research we showed that the desirable structure of the transition graph is an oriented rooted tree with orientation towards the root vertex, which corresponds to the normal class. It is clear that the transition graph of an arbitrary algorithm (policy) may not have this property. In this paper we study the structure of realistic transition graphs, which makes it possible to find classification inconsistencies, helping to transfer it into the desired form. The medical interpretation of ”dynamic treatment regime” considered in the article further clarifies the investigated framework.

\end{abstract}
\section{Introduction}

This paper considers a new mathematical formulation of a classification or pattern recognition problem, which differs from the traditional formulation of this problem [9,23,32] by its ideology. The classical classification problem from the theory of machine learning assumes the existence of a number of classes, which are known only at the level of some subsets of class elements [1,12,32]. Additional properties of these classes are assumed, for example - geometric compactness, smoothness or convexity etc. [2-5,10,33], or Occam's razor, and it is required to construct universal procedures capable of correctly classifying new objects by their classes [12,33]. The training stage of the algorithm by the set of available examples is basically a one-step process [1,32]. This step is followed by the testing and classification stages. 
In our application, which addresses a medical problem of treatment regimen adaptation [11] (and other similar problems), we pursue a different goal: - to put all objects into the same predefined class [6,7]. Depending on the state/class of the object, a certain action of this class is performed on it, and as a result, the object will move to another class. Transitions are not arbitrary and are limited by the state and action associated with it. In sequential action-transition steps, an object can transition between states, creating cycles and downtime, but can also transition to a target class in which no action is applied and no transition is performed. Our objective is to investigate conditions of models for successful (accurate) assignment of objects to the target class. This is in fact the study of the graph of transitions between classes. When actions are compound  or transition graphs are weighted [31], optimization problems arise about the best integrated policy. Our basic objective is to analyse the structure of graphs with one or more deterministic class actions, and we will analyze the properties of these graphs to transform the model into a suitable design.
In some sense this paper addresses a specific classification problem that does not fit into any of the known basic scenarios of supervised, unsupervised, or reinforced learning. However, being different from the traditional approach, Target Class Classification {\bf{TCC}} remains strongly associated with certain models of machine learning. The first association is with the concept of machine learning with reinforcement [28], which operates on a set of classes/states, and these classes and their elements make up the reinforcement learning environment. The agent (algorithm, recognizer) learns to move efficiently between classes in such a way as to optimize the target function about the final assignment of objects to a special target class (to achieve the goal). In this model, we are dealing with a dynamic, predictive rather than one-step classification problem.
Further, there are similarities with the classification of so-called unbalanced classes [21], where the need for overall accuracy of classification requires an emphasis on the classification of small classes (attention to the unseen). In contrast to this approach we are considering one regular class - the so-called target class [20,22], and seek for a procedure for assigning all objects to this target class (attention to the dominant class). The unbalanced classification itself has no algorithmic constraints and only appeals to the correct use of the whole training set. In contrast, our classification has algorithmic constraints that come from the subject domain (otherwise we would apply an action that transfers all objects into the target class in one step). 
It is to mention one more association that arises with sequential learning algorithms [13,17]. Sequential learning aims at narrowing the composition of classes, step by step, until the correct class for the object is found. In the case of our problem, sequential classification seeks to approximate the object being classified step by step to a single predetermined target class, and if the composition of classes is narrowed, it is only towards the target class. Here also the structure of classification is limited [34,35], it is based on the set of available data and rules of the subject domain [36,37], and it is required to optimize the strategy/algorithm of successful assignment of the object to the target class. 
We should note that the considered problem can be viewed from a different perspective - as a problem from the field of control theory or business process management theory. These theories, from our point of view, are effective in other circumstances, when functional environments are studied or when stochastic analysis is used. These theories, as a branch of mathematics use calculus, linear algebra, differential equations, linear programming and nonlinear optimization to analyze and design effective control systems. So we will use discrete mathematical instruments and we will stay in frame of terms of classification problems, because of our repetitive actions/transitions are linked to classes and because of the final objective is to classify to the target class.

An applied prototype of the {\bf{TCC}} concept is the well-known medical problem about the effective policy of treatment of chronic diseases [24,30]. The objectives of this group of medical approaches belong to the field of personalized (precision) medicine and are well known in terms of dynamic treatment regime (or adaptive treatment strategies) problems [24]. Known studies of this task focus on the statistical processing of treatment data from databases, databases that can be found also on the global web [11].

Section 1 brings Introduction to the research topic of the article. All the initial information necessary for the research in this paper is given in Section 2, where in Subsection 2.1 the research problem is defined, and in Subsections 2.2 and 2.3 the necessary concepts and definitions of the graph theory and supporting statements $P1-P5$ (by [14,19]) are given. Section 3 starts with mentioning two theorems from [7] on the strongly simplified version of the $TCC$ problem, and then presents two new results of this paper, Theorem 1 and Theorem 2, shedding light on the structure of the transition $TCC$ graph for the case where the transition graph is more realistic - with loops and with many outgoing links. The discussion and comments mention a way of using the results to identify and correct unnecessary defects in $TCC$ processes.

\section{{\bf{TCC}} concept and the related technique}

Target Class Classification, basically, is a data-driven procedure [18,24]. Suppose we need to make at most $k$ local sequential classifications, $a_1,a_2,...,a_k$ per object (a patient, in the medical treatment problem). $S_1$ denotes the initial class (pre-treatment information) whereas $S_j,1<j \le k$ denotes the intermediate outcome information available after decision $a_{(j-1)}$ and prior to decision $a_j$. Thus the time order of states and actions is $S_1,a_1,S_2,a_2,...,S_k,a_k$. This is an individual track, and the set of tracks over the population of individuals compose a trellis. Let $\bar{S_j} =\{S_1;\ldots;S_j\}$ denote the past and present information at time $j$. The primary outcome is $Y=u(\bar{S}_{k+1};\bar{a}_k)$ where $u$ is a known integrative function. 

Let us apply to an example from [24].
In the addiction management study, in medicine, $Y$ may be the main percent of days abstinent so $u$ counts the number of days abstinent and divides by treatment length in days. An adaptive treatment strategy is a sequence of decision rules, one per each individual decision point. Thus we denote an adaptive treatment strategy by the decision rules tuple $\{d_1,d_2,...,d_k\}$ where the decision rule $d_j$ takes the information available at time $j$. $\bar{S}_j=\{S_1,S_2,...,S_j\}$ and past treatment $\bar{a}_{(j-1)}=\{a_2,...,a_(j-1)\}$ outputs a treatment type/level, $a_j$. For example in the addiction management study, $k=2$ and the possible treatments, values for $a_1$ at the first decision time point are {\bf{med}} (drug) and {\bf{cbt}} (cognitive behavioral therapy) [24]. The information available for making the second decision includes pre-treatment information denoted by $S_1$, the first treatment, and the first intermediate outcomes denoted by $S_2$. A simple example of a decision rule, $d_2$, is: if the individual responds to initial treatment assign telephone monitoring with counselling or if the individual does not respond to initial treatment assign {\bf{em}} (enhanced motivational program) + {\bf{med}}+{\bf{cbt}}. The structure of the trellis can be very different. The initial states (classes) may be different, but in that case, it is possible to group tracks with the same initial state, which will simplify the preliminary analysis. The problem is also in tracks having different lengths, they may end in different states.
Practical data of this type, subject to the structure of transitions and tracks, may be available only partially or in aggregate format [18] due to privacy circumstances. Sometimes only the initial and final states of tracks and their lengths are given. Sometimes only transitions, - the previous state and the next state, or the number of these transitions are published from the data [26,27]. The available data, in turn, characterize the applied algorithms of state classification. Sequence profile type algorithms use state frequencies, while HMM profile class algorithms can use both first-order and higher-order transition frequencies, which is similar to the reinforcement learning and $MDP$ approaches [25,28]. The algorithm we consider below is the simplest for the simplest case of the problem. It considers first-order transitions, not chains, - and the result of the transition is defined by a single deterministic value. The problem is to understand the constraints for assigning objects to the target class. And if the practical information do not satisfy the constraints identified in the study, this indicates the need to verify, check data, and revise the decision-making strategy, or require the adoption of a new hypothesis about the processes and data.
Our objective is to optimize adaptive treatment strategies, that is, to create a treatment strategy that produces the best mean outcome value. A number of trials have been and are being conducted.
In order to name the main goals of the $TCC$, let us define the points of objectives [7]:
\begin{itemize}
\item
{\bf{Validation}} Given a transition graph, it is necessary to determine whether the corresponding process actually assigns all or most of the objects to the target class.
\medskip

\item
{\bf{Superiority}} Two policies are given, it is necessary to clarify which of them determines the process of the best (accurate) classification of objects into the target class.
\medskip

\item
{\bf{Optimization}} This is the main procedure of $TCC$. It looks tor a better composite policy by analysing the set of input policies or input data sets. Usually, the means of analyses is discrete mathematical analysis, the matrix analysis, Markov chains and HMM, etc.
\end{itemize}
	 
In most cases, the goal of these trials is to develop adaptive treatment strategies (treatment sequences for a particular person in a corresponding baseline) that are not ideal in principle. That is, the goal does not include confirmation that one adaptive treatment strategy is better than all others or standard treatment. Such trials should be followed by confirmatory trials comparing the optimized adaptive treatment strategy with an appropriate control experiment or standard treatment.

\subsection{Transition graphs, definitions, properties}
\subsubsection{Transition Graphs}
A transition graph $G=(V,E)$, abbreviated $TG$, is a collection of the following parts [16,31]:
\begin{enumerate}
\item 
	A finite set $V$ of states (vertices), that may be partitioned in three categories, among which there are initial states and final states. For the initial states, an activation procedure in envisioned.
\item 
	A set $A$ of possible actions that may be applied to these states.
\item 
	A finite set $E$ of transitions (edges) that show how to transition from some states to some others.
 \end{enumerate}
A normalized weight (probability) distribution over a finite set, $U$, is a function $F:U \to [0,1]$ such that $\sum_{u \in U}F(u)=1$. The support of this distribution $F$ is the set $support(F):=\{u \in U | F(u)>0\}$. For $\forall v \in V$ normalized weight distribution $F_v$ is defined on the set $N(v)=support(F_v)$ of neighbor vertices of $v$. An example of  graphical and algebraic presentations of a simple $TG$ is given in Fig. \ref{figtg}.

\begin{figure}
\includegraphics[width=\textwidth]{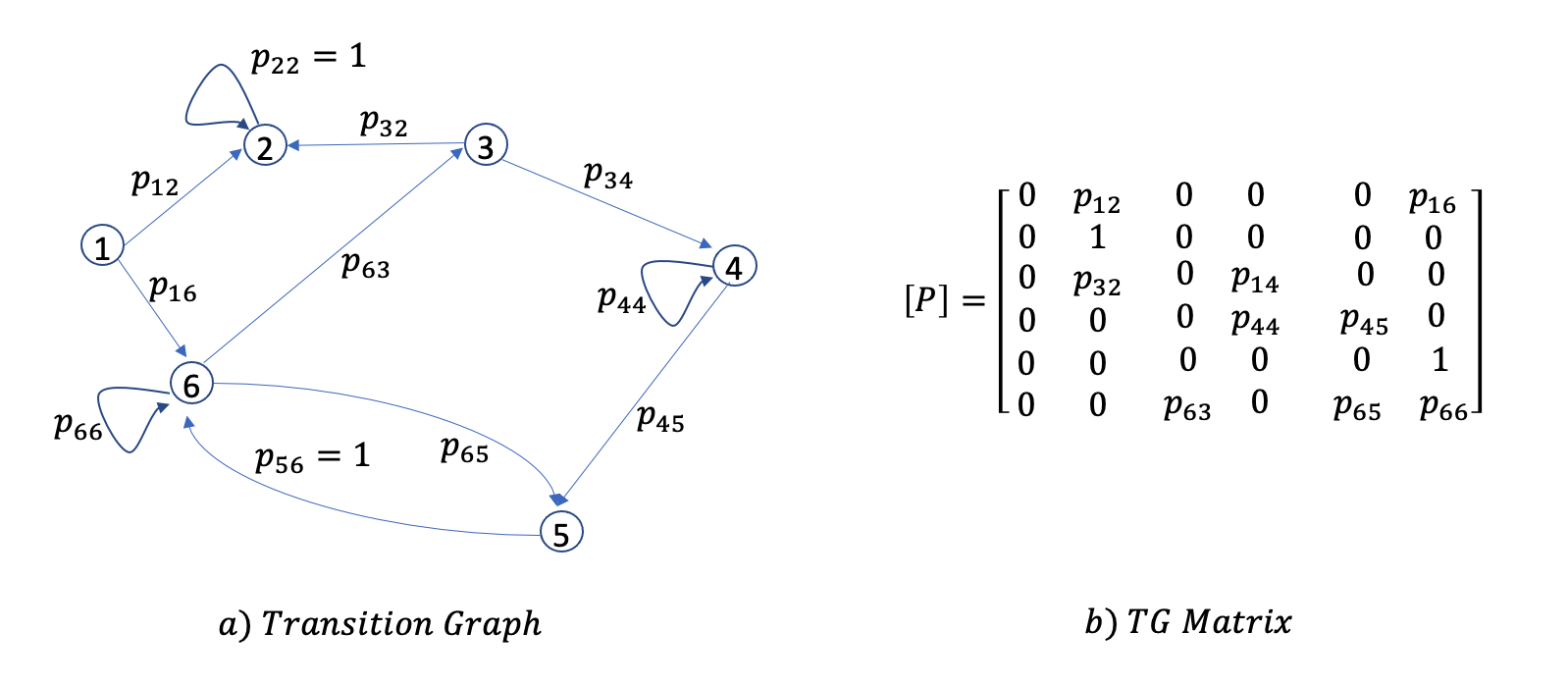}
\caption{An example TG in graphical form a) and in its algebraic form b). Rows in b) correspond to vertices of TG, and these rows sum to 1. A successful path through a transition graph is a series of edges forming a path beginning at some start state and ending at a final state. Concatenating the edges visited will yield the path string.} \label{figtg}
\end{figure} 

TGs were invented in 50’s. Today a plethora of TGs are defined and used in theories of Automata, in Markov chains and MDP, in Reinforcement learning and Control theory [16,25]. Problems addressing in this theories are related to the description and analysis of sequences of actions and time series (path analytics, recognition of languages); study of specific events, when they may happen and when not, with estimation of their chances; study of limit behavior of systems; likelihood estimates and satisfiability statistics; complexity type, and other issues. Historically, the automata theory developed entirely separately from the theory of stochastic processes and stochastic optimal control, with each developed by a separate mathematical community having distinct motives. It turns out, however, that there are fruitful connections between these fields. In particular, a number of classic infinite-state automata theoretic models, such as one-counter automata, context-free grammars, and pushdown automata, are in fact closely related to corresponding classic and well-studied countably infinite-state stochastic processes. Roughly speaking, such automata-theoretic models share the same (or, a closely related) underlying state transition system with corresponding classic stochastic processes.

Upon reflection, it should not be entirely surprising that this is the case. After all, Markov chains are nothing other than probabilistic state transition systems. In order for a class of infinite-state Markov chains to be considered important, it should not only model interesting real-world phenomena, but it should also hopefully be “analyzable” in some sense. Better yet, its analyses should have reasonable computational complexity. But these same criteria also apply to infinite-state automata-theoretic models: their relevance is at least partly dictated by whether we have efficient algorithms for analyzing them.

Clearly, we cannot devise effective algorithms for analyzing arbitrary finitely-presented countably infinite-state transition systems. For example, Turing machines are clearly finitely presented, but we cannot decide whether a Turing machine halts, i.e., whether we can reach the halting configuration from the start configuration. Furthermore, if we consider probabilistic Turing machines (PTMs), we easily see that there cannot exist any algorithm that computes any non-trivial approximation of the probability that a given probabilistic Turing machine halts.

\subsubsection{Directed Graphs}

When parallel edges and loops are admissible in an extension to the simple graph definition, we say that undirected pseudographs are given; pseudographs with no loops are multigraphs [8,15,19]. For a pair $u,v$ of vertices in a pseudograph $G$, $\mu G(u,v)$ denotes the number of edges between $u$ and $v$. In particular, $\mu G(u,v)$ is symmetric, and $\mu G(u,u)$ is the number of loops at $u$.
A directed graph (or just digraph) $D$ consists of a non-empty finite set $V(D)$ of vertices, and a finite set $A(D)$ of arcs – ordered pairs of adjacent vertices. In case of multiple arrows the entity is usually addressed as directed multigraph. $\mu G(u,v)$ is extended easily to the digraphs. Simple directed graphs are directed graphs that have no loops (arrows that directly connect vertices to themselves), and no multiple arrows with the same source and target nodes. Some authors describe digraphs with loops as loop-digraphs. Directed acyclic graphs (DAGs) are directed graphs with no directed cycles. Rooted trees are oriented trees in which all edges of the underlying undirected tree are directed either away from or towards the root. Some basic results about the digraphs may be found in [14,19], among them are:

{\bf{P1.}}	Every acyclic digraph has an acyclic ordering of its vertices; in this case vertices can be labelled $v_1,v_2,\cdots,v_n$ in a way that there is no arc from $v_i$ to $v_j$ unless $i<j$.

{\bf{P2.}}	Every acyclic digraph has a vertex of in-degree zero, as well as a vertex of out-degree zero.

A strong component of a digraph $D$ is a maximal induced subgraph of $D$ which is strong (strongly connected). If $D_1,D_2,\cdots,D_t$ are the strong components of $D$, then clearly and disjointly $V(D_1)\cup V(D_2)\cup \cdots 
\cup V(D_t)=V(D)$ (recall that a digraph with only one vertex is strong). The strong component digraph $SC(D)$ of $D$ is obtained by contracting the strong components of $D$ and deleting any parallel arcs obtained in this process. In other words, if $D_1,D_2,\cdots,D_t$ are the strong components of $D$, then $V(SC(D))=\{v_i | i \in [t]\}$ and $A(SC(D))=\{v_i v_j | (V(D_i),V(D_j))_D=\emptyset \}$.

{\bf{P3.}}	The subgraph of $D$ induced by the vertices of a dicycle in $D$ is strong, and hence is contained in a strong component of $D$. 

{\bf{P4.}}	$SC(D)$ is acyclic, the vertices of $SC(D)$ have an acyclic ordering.

An orientation of a simple graph $G$ is an oriented graph $H$ obtained from $G$ by replacing every edge $xy$ by either the arc $(x,y)$ or the arc $(y,x)$. The underlying graph $UG(D)$ of digraph $D$ is an undirected graph obtained from $D$ by replacing the set of arcs between $x$ and $y$ with one undirected edge $xy$. A digraph is connected if its underlying graph is connected. 

A digraph $D$ is an oriented tree, if $D$ is an orientation of a tree. A digraph $T$ is an out-tree (an in-tree) if $T$ is an oriented tree with just one vertex $s$ of in-degree zero (out-degree zero). The vertex $s$ is the root of $T$. If an out-tree (in-tree) $T$ is a spanning subgraph of $D$, $T$ is called an out-branching (an in-branching).

{\bf{P5.}}	A connected digraph $D$ contains an out-branching (in-branching) if and only if $D$ has only one initial (terminal) strong component.

In the graph theory, a cactus is a connected graph in which any two simple cycles have at most one vertex in common. Equivalently, it is a connected graph in which every edge belongs to at most one simple cycle. Similarly, a directed cactus is a strongly connected digraph in which each arc is contained in exactly one cycle. The construction we will consider below is composed of a set of oriented cycles, thus, being a strongly connected component of some initial graph, plus a number of branches directed towards the cycles.

\section{$TCC$ Graph Properties}
$TCC$ graphs are transition graphs, which in general are digraphs. Depending on the type of transitions in $TCC$ - deterministic or stochastic, one or many actions, transition branching, - the structure of this graphs obtain additional properties becoming a narrower subclass of transition digraphs. Simple deterministic $TCC$ ($sdTCC$) is when transition from any class-action pair is to a unique class: $(c_i,a_(c_i )) \to c_j$. In such $TCC$ graph, all but one vertex have exactly one outgoing arc (in simple digraph this can’t be a loop). Vertex $v_0$, corresponding to the target class, have out-degree $0$. Initial properties of $sdTCC$ graphs (given in [7]) are as follows:

{\bf{T1.}} If graph $G$ of a $sdTCC$ model is connected (semi-connected digraph), then $G$ is a tree with the root at vertex $v_0$ and with edges oriented from the terminal/leaf, as well as internal vertices of the tree towards the direction of the vertex $v_0$.

{\bf{T2.}} Arbitrary graph $G$ of a $sdTCC$ consists of one in-branching tree, rooted at $v_0$, and several other connected components structured as one cycle cactus graph. The cactus cycle is oriented, and the tree-like components are connected to cycle vertices having orientation to the cycle.

In this paper, we  consider extension of these assertions to other classes of $TCC$ graphs. The first case is that of $TCC$ with loops. In terms of the application problem, a loop is the case of classification-transition to the same class. We differentiate three types of oriented graphs:  graph of type $v_0$, graph of type "loop", and graph of type "cactus", as shown in Fig.\ref{fig3}.

\begin{theorem}
The structure of a graph $sdTCC$ with loops allowed includes one connected component of type $v_0$, and possibly, several components  of types "loop", and "cactus".
\label{3g}
\end{theorem}

\begin{proof}
Proof of this theorem is similar to the proof of T1 and T2 given in [6,7]. $TCC$ objective is satisfied when the graph consists of only one component of type $v_0$ (no more vertices or components outside of this component). Fig. \ref{fig3} clearly indicates the possible defects when there are more than one component. The root cause of all defects is disconnectedness. Oriented cycles and loops cause unnecessary or unimportant transitions between classes.

\begin{figure}
\includegraphics[width=\textwidth]{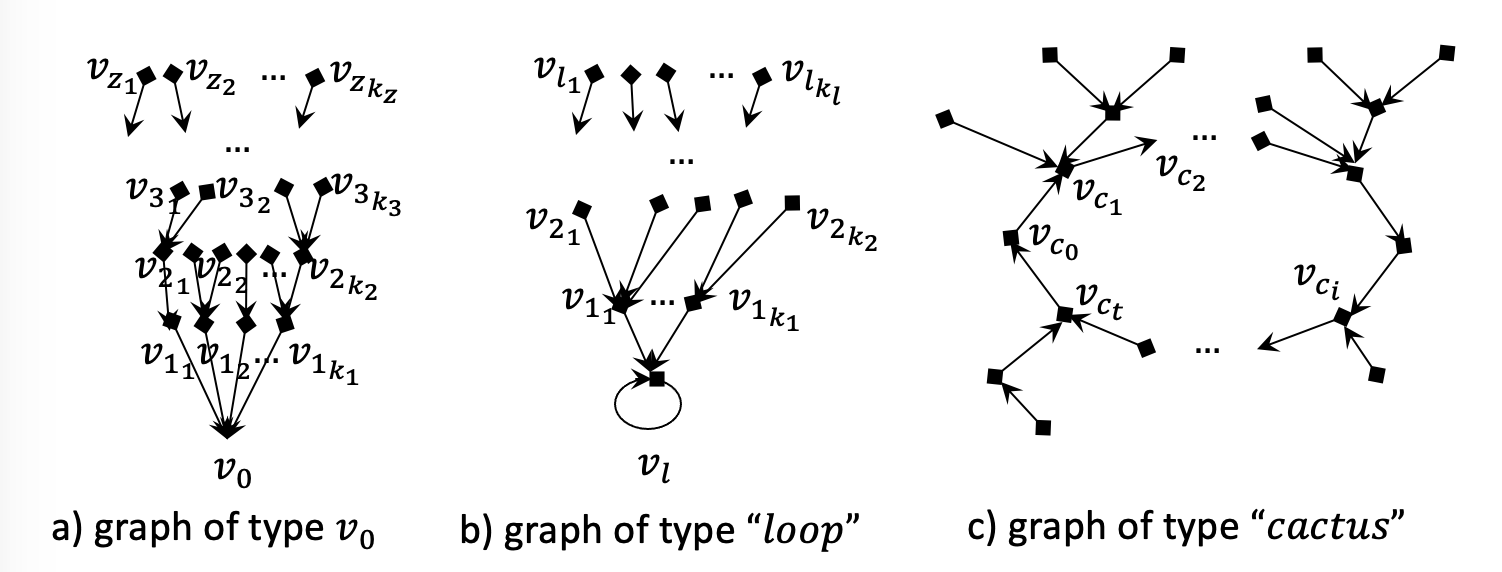}
\caption{Graph types that may appear as components of general $sdTCC$ (loops allowed). Graph of type $v_0$ is an in-branching tree to the vertex $v_0$ that corresponds to the target class. Graph “loop” is similar to case a) and their number in $sdTCC$ is equal to the number of  “loop” vertices. The reminder components of $sdTCC$ have structure of simple cactus graphs with one oriented cycle with incoming tree-type fragments like one given in c). Number of vertices and structural particularities at components can be diverse.} \label{fig3}
\end{figure} 

Special mention should be made of the role of loops at the vertices of the graph under the assumption that they are admissible. It is necessary to distinguish between isolated and non-isolated loops. The former stand apart and do not affect the overall structure of the graph. The second group of loops also do not introduce additional difficulties, they just behave similar to the vertex $v_0$.

For Theorem \ref{3g} we just prefer to bring the sketch of the proof. Consider an arbitrary vertex $u$ and follow by its outgoing edge $e_u$ to the end point $v$ of this edge. Continue the path by the edge $e_v$ from $v$ and so on, and in this way valid oriented path $u,e_u,v,e_v,...$ will be constructed. The number of vertices and edges is finite so that at some step we will come back to some vertex of the path. At this point we obtain one oriented cycle and, may be, that an oriented path, that goes to the vertex that formed the cycle. This means that there may be several oriented paths that arrive into the vertices of the cycle. Paths that enter into the same vertex of cycle may intersect with each other forming a tree like fragment (point c) of the Fig. \ref{fig3}).

This process of forming cycles may be interrupted in two cases. These are when the path considered above arrives to one of the loop vertices, or, when it enters into the vertex $v_0$. Again, the paths that enter to $v_0$ or to a particular loop vertex, may intersect forming an oriented tree to these special vertices (points a) and b) of the Fig. \ref{fig3}). This is the whole sketch of the proof of this theorem.  $\square$
\end{proof}

$TCC$ graph as a digraph may have more of the properties of general digraphs. We outline specific properties that are related to the graph of our applied medical TCC problem.

Thus we need to study the structure of $TCC$ graphs in supposition that it is a digraph with: - no multiple edges, but loops, and one or more outgoing arcs from vertices except the $v_0$. By this, clearly, we consider the case of $TCC$ with several classifiers. The role of this model is twofold. Imagine several separate classifiers $A_1,A_2,\cdots,A_m$. From the point of view of the applied task, this is a group of medical institutions, each of which operates according to its own standard (policy) for prescribing procedures and with the practice of transition to other states. From the point of view of the $TCC$ graph, in this case one or more edges emanate from a vertex to other classes ($mcTCC$). The problem in this case is the following: which of these strategies is the best one, or is it possible to combine the available strategies into one single and optimal strategy. 

The same $mcTCC$ graph also arises within the framework of the stochastic model of TCC (which we do not consider in the paper), as follows. Suppose that actions in classes are assigned based on probabilities (full probabilities), or transitions are made according to some probabilities. This leads to a graph, with probability-weighted edges coming out of its vertices. Eliminating edges of zero probability, and considering  the simple graph with all other edges, we arrive to the same $mcTCC$ graph which models the logical level structure of the possible transitions.

\begin{theorem}
The structure of $TCC$ graph with many classifiers ($mcTCC$) obeys the structural characterisations given in points $\alpha)-\epsilon)$ below, summarised in Fig.3.
\end{theorem}

\begin{proof}
$\alpha)$ Let us took the vertex $v_0$ of the target class, and similar to the reasoning of $T1$ and Theorem 1, construct the component $G_0$ of graph $G$, where $v_0 \in G_0$. It is asserted that some structure like in Fig. 2 a). is contained in graph $G_0$. 
The first layer of this structure (vertices with indexes $1_1,1_2,\cdots,1_{k_1}$ in the Fig. \ref{figv0}), contains all vertices from which there exists a transition to $v_0$, i.e. from the reminder part of the graph the transition to $v_0$ is already excluded. In this step and in further steps,  all top-down transitions except those that are used in the constructed tree are excluded. Let us call these (already used) connections real. 

However, in contrast to T1., in this new case of $mcTCC$ bottom-up transitions are also possible, for example, from layer $1$ to layer $3$ (see Fig. \ref{figv0}), which can make cycles. Two cases are possible: one when link goes from a lower layer vertex $u$ to a higher layer vertex $v$, so that the oriented path from $v$ to $v_0$ passes through $u$ and an oriented cycle is formed, and the opposite case, when the forming of cycle is not immediate, and it depends on the structure of down-up edges concerned to the vertex $v$. Any such transition creates a problem of assigning the object to the normal class. With an oriented cycle, transitions can become an infinite process without reaching $v_0$.

Another also new and important group of transitions can lead from vertices of $G_0$ different from $v_0$, to vertices of graph $G-G_0$ that are outside the component of $G_0$. These two groups of connections, which are possible but not necessary, we call  virtual ones – when they exist.

\begin{figure}
\includegraphics[scale=0.4]{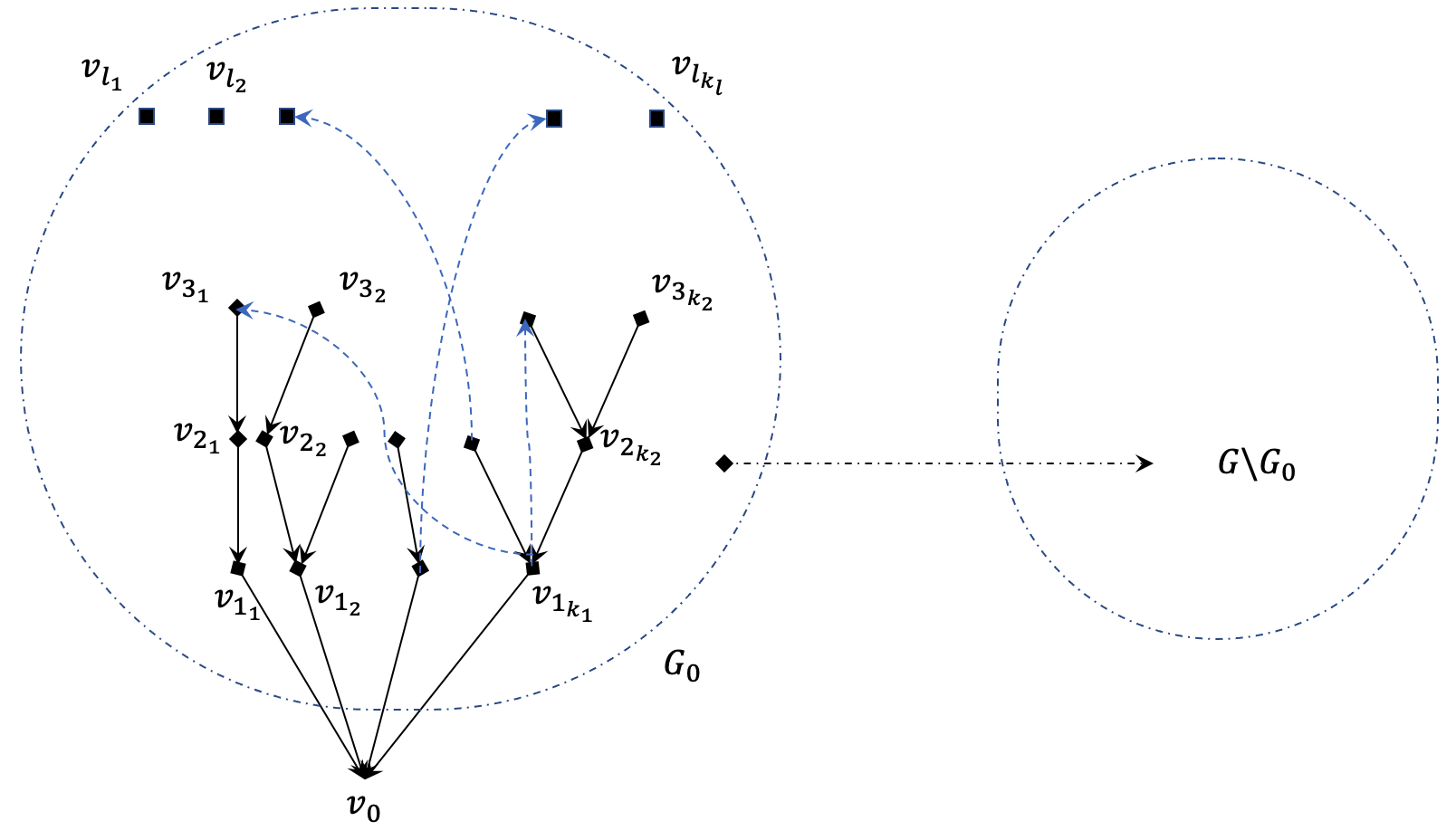}
\caption{The core part of target class classification graph structure and its relation to the complementary part of the structure.} \label{figv0}
\end{figure}  
\bigskip
$\beta)$ First, let us distinguish all vertices with loop, that are isolated from other vertices of the graph. Obviously, such vertices represent dangling components of connectivity and are as if outside the main structure of the graph. 
Further, let us distinguish the class of loops in vertices from which no regular (no loop) edges come out. These vertices, as we can easily check, initiate components similar to component $G_0$ of $v_0$. Let us denote their united graph by $G_l$. It is important to note that the components of type $a)$ and $b)$ from Fig. \ref{fig3} which also arise in this case are not strongly connected components. This is explained by the fact that no edges outgo from the roots of such trees and no edges enter into the terminal vertices. At the same time, these components are not isolated, and connections with other components can come from them. As for other vertices, that are with loops, they behave like other usual (non-loop) vertices in the structure of the whole graph.
\\

$\gamma)$ Now consider an arbitrary vertex $x$ of the graph $G-G_0-G_l$. At least one edge originates from $x$, let it be an edge $e$ that leads to some vertex $y$. Vertex $y$ can belong only to $G-G_0-G_l$. 
Continuing the started chain step by step, as before, and we will reach an intersection with the path at some point $z$, forming an oriented cycle. We are talking about a nondegenerate cycle, i.e., not a cycle consisting of a single vertex (loop vertex). This is true, because at this stage there is always an outgoing regular edge from a vertex $z$ and we are talking about closing the path of this edge. The segment from $x$ to $z$ if $x \neq z$ remains outside the cycle and forms a branch to it, being oriented towards the cycle. This holds for all vertices $x$ belonging to $G-G_0-G_l$. Here too, we refer to the entire chain structure as real, and the branches outside the chain as virtual.
\\

$\delta)$ Consider all nondegenerate cycles that are composed in the constructions of point $\gamma)$ over all vertices $x \in G-G_0-G_l$. Some pairs of these cycles may intersect. Intersection means existence of a common vertex but even more intersection does not prevent further reasoning that we do. Note that our cycles are oriented (dicycle) and that an oriented cycle is a strongly connected subgraph (see $P3.$). Two cycles with a common vertex also constitute a strongly connected structure. This construction could be extended by new dicycles when they have intersection with at least  one of the initial two cycles. If we define the equivalence relation over pairs of cycles $(p,q)$ as the fact of existence of chain of pairwise intersecting cycles from the first cycle $p$ to the second cycle $q$, then this equivalence relation will partition the set of all cycles of step $\gamma)$ into non-intersecting clusters. These clusters $\theta_1,\theta_2,\ldots,\theta_k$ are part of the so-called strong component digraph $SC(G-G_0-G_l)$ inside the graph $G-G_0-G_l$. $SC(G-G_0-G_l)$ is constructed by contracting the clusters $\theta_i$ into vertices $t_i$, and deleting any parallel arcs obtained in this process. Other components are $1$-vertex strong connections.
\\

$\epsilon)$ It is known (see $P1.$ and $P4.$) that each strong component digraph $SC(G-G_0-G_l)$ has acyclic ordering of its vertices. This implies that the strong components included
in $G-G_0-G_l$ can be labelled as $\theta^1,\theta^2,\ldots,\theta^l,\theta_1,\theta_2,\ldots,\theta_k$ in a way, 
that there is no arc from $\theta_j$ to $\theta_i$, as well as from $\theta^j$ to $\theta^i$  unless 
$j < i$. 
Here $\theta^1,\theta^2,\ldots,\theta^l$ correspond to $1$-point $SCC$. We call such an ordering an acyclic ordering of the strong nondegenerate components of $G-G_0-G_l$. Thus the following schematic representation of graph $TCC$ is valid.

 \begin{figure}
\includegraphics[width=\textwidth]{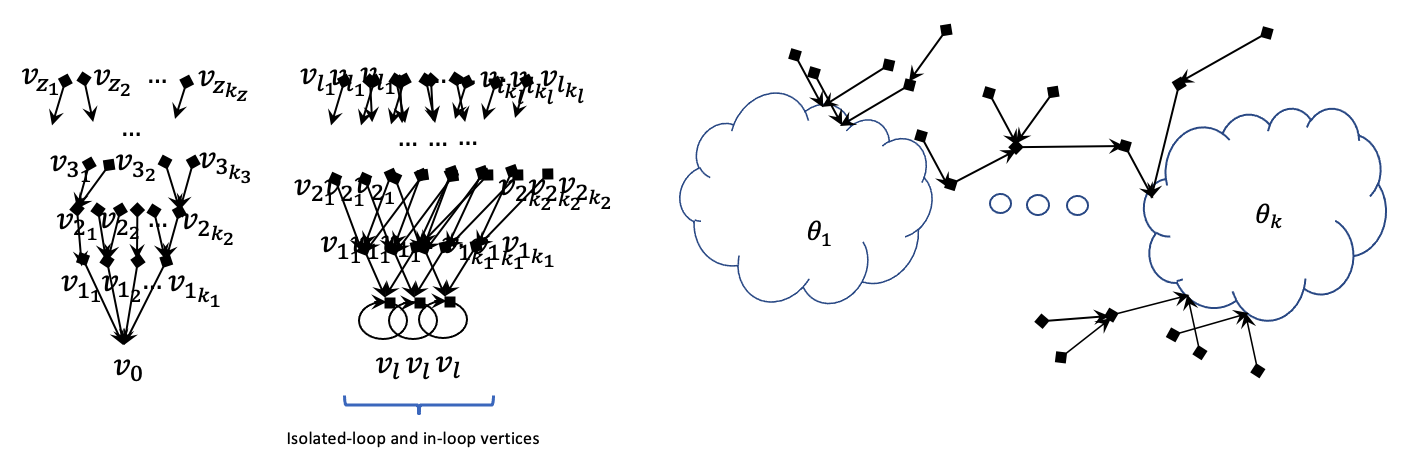}
\caption{$TCC$ graph structure split to components: $v_0$, isolated-loop and in-loop, and the part $SC(G-G_0-G_l)$.} \label{figstr}
\label{figg}
\end{figure}
\end{proof}

\section{Conclusion}
An applied problem of medical area is considered that shows that among the diversity of classification algorithms, there is a practical need to develop new algorithms that assign objects to one predetermined class through the sequence of classifications and transformations. This can be considered as a control problem, or a management problem, however, this is also a particular discrete mathematical problem modelled with the transition graphs. Sequential steps of classification consist of definition of the class of object, application of an action attached to the class of object, and transition to the updated class of this object. On a practical medical level, the study of such algorithms is related to the area of precision/personalised medicine. This paper describes admissible structures of graphs of transitions between classes, depending on properties of actions of classes - deterministic, multivariate, stochastic. These properties determine the structure of transition graphs that allow us to reveal its defects, knowing how to improve its current structure with achieving the goal of the $TCC$ process - classification to the target class. \\

{\bf{Acknowledgments}} 

This paper is partially supported by by grant №21SC-BRFFR-1B029 of the Science Committee of MESCS RA.


\begin{thebibliography}{8}

\bibitem{Aslanyan1975}
Aslanyan, L.: On a recognition method, based on partitioning of classes by the disjunctive normal forms. Kibernetika \textbf{5}, 103--110 (1975)

\bibitem{Aslanyan1978}
Aslanyan, L., Karakhanyan, V., Torosyan B.: On compactness of subsets of n-dimensional unite cube. In Akad. Nauk SSSR Dokl. \textbf{19}, 781--785 (1978)

\bibitem{Aslanyan1979}
Aslanyan, L.: The discrete isoperimetry problem and related extremal problems for discrete spaces. Problemy kibernetiki \textbf{36}, 85--128 (1979)

\bibitem{Aslanyan2001}
Aslanyan L., Zhuravlev Yu.: Logic separation principle. Computer Science and Information Technologies Conference, Yerevan, September 17-20, 151--156 (2001)

\bibitem{Aslanyan2017}
Aslanyan, L., Sahakyan, H.: The splitting technique in monotone recognition. Discrete Applied Mathematics \textbf{216}, 502--512 (2017)

\bibitem{Aslanyan2021}
Aslanyan, L., Krasnoproshin, V., Ryazanov V., Sahakyan, H.: Logical-combinatorial approaches in dynamic recognition problems. Mathematical Problems of Computer Science \textbf{54}, 96--107 (2010)

\bibitem{Aslanyan2022}
Aslanyan, L., Gishyan, K., Sahakyan, H.: Deterministic Recursion in Target Class Classification. Proceedings of the 13th Conference on Data analysis methods for software systems, Vilnius University Proceedings \textbf{31}(6), 6 (2016), doi: 10.15388/DAMSS.13.2022

\bibitem{Bollobas1998}
Bollobas, B.: Modern graph theory. Vol. 184, Springer Science and Business Media (1998)

\bibitem{Bongard1967}
Bongard, M.: Problem of Cognition (in Russian), Fizmatgiz, Moscow, (1967)

\bibitem{Braverman1962}
Braverman, E.: Experiments on machine learning to recognize visual patterns. Automation and Remote Control, \textbf{23}, 315–327 (1962)

\bibitem{Chakraborty2013}
Chakraborty, B.: Statistical methods for dynamic treatment regimes. Springer-Verlag, doi: 10:978–1 (2013)

\bibitem{Chervonenkis1999}
Chervonenkis, Ya., Vapnik, V.: Pattern recognition Theory (in Russian). Nauka, Moscow (1974)

\bibitem{Even-Zohar2001}
Even-Zohar, Y., Roth, D.: A sequential model for multi-class classification. arXiv preprint cs/0106044 (2001)

\bibitem{Godsil2001}
Godsil, C. and Royle G.: Algebraic graph theory. volume 207, Springer Science and Business Media (2001)

\bibitem{Harary1069}
Harary, F.: Book title. Graph Theory, Addison-Wesley, Reading, MA (1969)

\bibitem{Hopcroft2001}
Hopcroft, J., Motwani, R., Ullman, J.: Introduction to automata theory, languages, and computation. Acm Sigact News, \textbf{32}(1), 60--65 (2001)

\bibitem{Hu2017}
Hu, Yi-Qi, Hong, Qian, Yang, Yu: Sequential classification-based optimization for direct policy search. In Proceedings of the AAAI Conference on Artificial Intelligence, vol. 31, no. 1 (2017)

\bibitem{Johnson2016}
Johnson, A., Pollard, T., Shen L., et al.: MIMIC-III, a freely accessible critical care database. Sci Data 3, 160035 (2016). https://doi.org/10.1038/sdata.2016.35

\bibitem{Bang-Jensen2018}
Bang-Jensen, J., Gutin, G. eds.: Classes of directed graphs, Cham: Springer (2018)

\bibitem{Khan2014}
Khan, S., Madden, M.: One-class classification: taxonomy of study and review of techniques. The Knowledge Engineering Review \textbf{29}(3), 345-374 (2014)

\bibitem{Koko2013}
Koko, S., Capponi, C.: On multi-class classification through the minimization of the confusion matrix norm. IIn Asian Conference on Machine Learning, pp. 277–292. PMLR (2013)

\bibitem{Leng2015}
Leng, Qian, Qi, Honggang, Miao, Jun, Zhu, Wentao: One-class classification with extreme learning machine. In: Mathematical problems in engineering, (2015)

\bibitem{Mohri2018}
Mohri, M., Rostamizadeh, A., Talwalkar, A.: Foundations of machine learning, MIT press (2018)

\bibitem{Murphy2005}
Murphy, S.: An experimental design for the development of adaptive treatment strategies. Statistics in medicine, \textbf{24}(10), 1455–-1481 (2005)

\bibitem{Rabiner1989}
Rabiner, L.: A tutorial on hidden markov models and selected applications in speech recognition. Proceedings of the IEEE, 77(2):257–286 (1989)

\bibitem{Sahakyan2006}
Sahakyan, H., Aslanyan L.: Numerical characterization of n-cube subset partitioning. Electronic Notes in Discrete Mathematics \textbf{27}, 3--4 (2006)

\bibitem{Sahakyan2009}
Sahakyan, H.: Numerical characterisation of n-cube subset partitioning. Discrete Applied Mathematics \textbf{157}(9), 2191–-2197 (2009)

\bibitem{Sutton1999}
Sutton, R., Barto, A.: Reinforcement learning: An introduction. MIT press (2018)

\bibitem{Vaintsvaig1973}
Vaintsvaig, M.: Pattern recognition learning algorithm kora. Moscow, Sovetskoe Radio, 110-–116 (1973)

\bibitem{Zhang2019}
Zhang, Z. et al.: Reinforcement learning in clinical medicine: a method to optimize dynamic treatment regime over time. Annals of translational medicine \textbf{7}(14), 99--110 (2019)

\bibitem{Zhou2012}
Zhou, X., Yang, C., Gui, W.: State transition algorithm. Journal of Industrial and Management Optimization \textbf{8}(4), 1039--1056 (2012)

\bibitem{Zhuravlev1971}
Zhuravlev, Yu., Nikiforov, V.: Recognition algorithms based on computation of estimatese. Kibernetika \textbf{7}(3), 387–-400 (1971)

\bibitem{Zhuravlev1998}
Zhuravlev, Yu.: Selected scientific works. Moscow, Magistr (1998)

\bibitem{Zhuravlev2014}
Zhuravlev, Yu., Aslanyan, L., Ryazanov, V.: Analysis of a training sample and classification in one recognition model. Pattern recognition and image analysis \textbf{24}(3), 347--352 (2014)

\bibitem{Zhuravlev2017}
Zhuravlev, Yu., Aslanyan, L., Ryazanov, V., Sahakyan, H.: Application driven inverse type constraint satisfaction problems. Pattern recognition and image analysis \textbf{27}(3), 418--425 (2017)

\bibitem{Zhuravlev2019}
Zhuravlev, Yu., Ryazanov, V., Aslanyan, L., Sahakyan, H.: On a classification method for a large number of classes. Pattern Recognition and Image Analysis \textbf{29}(3), 366--376 (2019)

\bibitem{Zhuravlev2020}
Zhuravlev, Yu., Ryazanov, V., Ryazanov, Vl., Sahakyan H.: Comparison of different dichotomous classification algorithms. Pattern Recognition and Image Analysis \textbf{30}(3), 303--314 (2020)
\end{thebibliography}
\end{document}